\documentclass{ecai} 

\usepackage{latexsym}
\usepackage{amssymb}
\usepackage{amsmath}
\usepackage{amsthm}
\usepackage{booktabs}
\usepackage{enumitem}
\usepackage{graphicx}
\usepackage{color}
\usepackage{algorithm}
\usepackage{algpseudocode}
\usepackage{subcaption}
\usepackage{float}

\newtheorem{theorem}{Theorem}
\newtheorem{lemma}[theorem]{Lemma}
\newtheorem{corollary}[theorem]{Corollary}

\newtheorem{definition}{Definition}

\newcommand{\BibTeX}{B\kern-.05em{\sc i\kern-.025em b}\kern-.08em\TeX}

\begin{document}


\begin{frontmatter}


\paperid{9586} 


\title{"So, Tell Me About Your Policy...": \\ Distillation of interpretable policies from Deep Reinforcement Learning agents}


\author[A]{\fnms{Giovanni}~\snm{Dispoto}\thanks{Corresponding Author. Email: giovanni.dispoto@polimi.it.}}
\author[A]{\fnms{Paolo}~\snm{Bonetti}}
\author[A]{\fnms{Marcello}~\snm{Restelli}} 

\address[A]{Politecnico di Milano, Dipartimento di Elettronica, Informazione e Bioingegneria}


\begin{abstract}
Recent advances in Reinforcement Learning (RL) largely benefit from the inclusion of Deep Neural Networks, boosting the number of novel approaches proposed in the field of Deep Reinforcement Learning (DRL). These techniques demonstrate the ability to tackle complex games such as Atari, Go, and other real-world applications, including financial trading. Nevertheless, a significant challenge emerges from the lack of interpretability, particularly when attempting to comprehend the underlying patterns learned, the relative importance of the state features, and how they are integrated to generate the policy's output. For this reason, in mission-critical and real-world settings, it is often preferred to deploy a simpler and more interpretable algorithm, although at the cost of performance. In this paper, we propose a novel algorithm, supported by theoretical guarantees, that can extract an interpretable policy (e.g., a linear policy) without disregarding the peculiarities of expert behavior. This result is obtained by considering the advantage function, which includes information about why an action is superior to the others. In contrast to previous works, our approach enables the training of an interpretable policy using previously collected experience.
The proposed algorithm is empirically evaluated on classic control environments and on a financial trading scenario, demonstrating its ability to extract meaningful information from complex expert policies.
\end{abstract}

\end{frontmatter}


\section{Introduction}
In recent years, the integration of deep neural networks as function approximators within the Reinforcement Learning (RL) paradigm \citep{sutton1998reinforcement} --- commonly referred to as Deep Reinforcement Learning --- has led to groundbreaking achievements across a variety of domains. These range from mastering complex games, such as Atari \citep{DBLP:journals/corr/MnihKSGAWR13} and Go \citep{44806}, to addressing real-world problems, including financial trading \citep{DBLP:journals/corr/abs-2101-07107} and autonomous driving \citep{Sallab_2017}.
The strength of Deep RL lies in its ability to learn sophisticated control policies directly from high-dimensional data. However, this expressive power comes at the cost of interpretability. Deep neural networks are often considered \emph{black-box} models, offering limited insight into the reasoning behind their decisions.

This lack of transparency poses a significant barrier to the deployment of Deep RL in high-stakes and safety-critical environments, where understanding and verifying agent behavior is essential.
In such contexts, practitioners may opt for simpler yet interpretable policy representations -- such as decision tree policies \citep{DBLP:journals/corr/abs-1807-05887} or programmatic policies \citep{LIKMETA2020103568} -- even when it entails sacrificing raw performance.
The emerging field of Explainable Reinforcement Learning (XRL)~\citep{DBLP:journals/corr/abs-2006-11371} aims to bridge this gap by developing methods that make RL agents more interpretable, either by elucidating the decisions of complex models or by learning inherently transparent policies.

In this work, we address the challenge of interpretability by introducing a novel approach based on policy distillation, i.e., exploiting the available complex model to transfer its knowledge to a simpler linear model. The approach we propose enhances interpretability without disregarding the peculiarities of expert behavior.
Specifically, we focus on interpreting the behavior of a pre-trained expert policy $\pi_E$ by training a surrogate model --- such as a linear policy --- that is both transparent and faithful to the expert's decision making. Unlike traditional Behavioral Cloning methods \citep{DBLP:journals/corr/abs-1011-0686,DBLP:journals/corr/HoE16}, which solely imitate the expert's actions, our approach incorporates information from the advantage function, which naturally arises from the theoretical analysis, to guarantee the control of the loss of information in terms of value functions. This allows the surrogate model to capture not only a similar distribution over actions uniformly, but also to focus on specific state-action pairs that are crucial to the computation of the expected return. This is crucial when, as usually happens, the interpretable policy space is less representative w.r.t. the expert policy space; thus, only a portion of the expert policy can be imitated by the surrogate one, and we need to focus on the most important aspects of the expert to achieve similar returns. This distinction is particularly relevant in scenarios where multiple actions may appear equally viable, but subtle differences in expected future reward make one action decisively superior.

A significant real-world motivational example could be retrieved from a financial trading setting, where most of the time there are no large price oscillations, and it is similar, in terms of profit, to buying or not buying a certain asset. However, at certain instants, large price changes could occur, and it is crucial that the simpler policy learned from the expert is able to perform the same action when most of the profit appears. 

To summarize, our contributions are twofold:
\begin{enumerate}
    \item we propose an interpretable policy distillation framework that leverages advantage-based imitation, with theoretical guarantees;
    \item we demonstrate through empirical evaluation that our method can outperform standard Behavioral Cloning in replicating expert behavior, particularly in complex tasks such as financial trading, while providing human-understandable insights into the underlying decision process.
\end{enumerate}

The remainder of this paper is organized as follows. In Section~\ref{Sec:Preliminaries}, preliminaries on Reinforcement Learning are presented, and the notation is introduced. In Section~\ref{Sec:rel}, some related works in interpretable and explainable RL are discussed. In Section~\ref{Sec:Bound} we formulate the bound that is optimized in the training of the surrogate model, and Section~\ref{Sec:Bound:Subsec:Algorithm} details the proposed algorithm. In Section~\ref{Sec:Experiments}, we empirically validate the algorithm, showing the ability to outperform Behavioral Cloning in specific tasks. Moreover, focusing on a real-world financial trading scenario that motivates the work, we demonstrate how the surrogate linear policy uncovers the behavior of a more complex expert policy. Finally, Section~\ref{Sec:Conclusions} concludes the paper and outlines future research directions.

\section{Preliminaries}
\label{Sec:Preliminaries}

Sequential decision-making problems are often formalized via the concept of Markov Decision Processes~\citep[MDP,][]{sutton1998reinforcement}.
An MDP is a tuple $\mathcal{M} =(\mathcal{S},\mathcal{A},\mathcal{P},\mathcal{R},\gamma)$ where $\mathcal{S}$ is the state space, $\mathcal{A}$ is the action space, $\mathcal{P}$ is the transition model $\mathcal{P}:\mathcal{S}\times \mathcal{A}\rightarrow P(\mathcal{S})$ that gives the probability distribution of the next state given the current state and the action performed, $\mathcal{R}$ is the reward function $\mathcal{R}: \mathcal{S} \times \mathcal{A} \rightarrow \mathbb{R}$ and $\mathcal{\gamma} \in [0, 1)$ is the discount factor.
In this context, an agent is characterized by a policy $\mathcal{\pi}$ that maps to each state a distribution over actions $\mathcal{\pi}: \mathcal{S} \rightarrow \mathbb{P}(\mathcal{A})$.
Given a policy $\mathcal{\pi}$, we can define the \textit{Value Function} $V^\mathcal{\pi}(s) = \mathbb{E}_{\pi}[\sum_{t = 0}^\infty \gamma^t r_t | s_0 = s, a_t \sim \pi(s_t)] $ which represents the expected discounted return from the state $s$, following the policy $\pi$. Moreover, we can also define the \textit{State-Action Value Function} $Q^\mathcal{\pi}(s, a) = \mathbb{E}_{\pi}[\sum_{t = 0}^\infty \gamma^t r_t | s_0 = s, a_0 = a, a_t \sim \pi(s_t)] $ that represents the expected return from state $s$ if we take action $a$ and then we follow the policy $\pi$. Finally, the advantage function of a policy $\pi$ is defined as $A^\pi(s,a) = Q^\pi(s,a) - V\pi(s)$, which represents the "advantage" of taking the action $a$ in state $s$ w.r.t. selecting an action according to the distribution induced by $\pi$.

Solving the MDP means finding the \textit{optimal} policy $\pi^*$, which is the policy that is able to achieve the maximum expected discounted return in each state $\pi^* = \max_\pi V^\pi(s), \forall s \in \mathcal{S}$. In this context, Reinforcement Learning (RL) algorithms are used when the transition model $\mathcal{P}$ and the reward function $\mathcal{R}$ are unknown or are too large to be stored and used for computations. The two main tasks characterizing RL algorithms are predictions and control. In prediction, given a policy $\pi$, the objective is to obtain its value function $V^\pi(s)$, therefore estimating its capability to get significant returns. On the contrary, in control tasks, the goal is to find the optimal policy $\pi^*$ that maximizes the expected return.

RL algorithms can be classified using different criteria. In this paper, we mainly focus on three categories, depending on the components that they parametrize and optimize: Value-based, Policy-based, and Actor-Critic (AC).
In Value-based RL,  the (action) value function is estimated and the policy is subsequently extracted, taking for each state the action with the maximum expected return $\pi^*(s) = \arg\max_a Q^{\pi^*}(s, a), \forall s \in \mathcal{S}$. 
On the other hand, policy-based algorithms are directly based on a parametrization of the policy, searching for the optimal policy with an optimization directly performed in the policy space, often exploiting the Policy Gradient Theorem~\citep{10.5555/3009657.3009806}. Being $J(\theta)$ the expected return, the Policy Gradient theorem allows to express the gradient of the expected return in a convenient fashion, which can be optimized in practice with different strategies (e.g., REINFORCE~\citep{10.1007/BF00992696}):

\begin{equation}
\label{eq:pg}
    \nabla_\theta J(\pi_\theta) = \mathbb{E}_{\tau \sim \pi_\theta}[\sum_{t= 0}^{T} \nabla_\theta \text{log}\pi_\theta(a_t|s_t)A^{\pi_\theta}(s_t, a_t)],
\end{equation}

where $\pi_\theta$ is a parametric policy.

Finally, Actor-Critic (AC) algorithms (e.g., SAC~\citep{DBLP:journals/corr/abs-1801-01290}) combine the two approaches, being composed of a Critic,  which estimates the $Q_w(s,a)$ function, and an Actor, which updates the parametric policy $\pi_\theta$ in the direction suggested by the critic.

\section{Related Works}
\label{Sec:rel}
 In Machine Learning, the use of complex function approximators, such as large neural networks with millions or even billions of parameters, is typically associated with improved performances on complex tasks due to their large hypothesis space that allows them to represent complex functions \citep{DBLP:journals/corr/abs-2001-08361}. On the other hand, this poses a significant challenge in terms of understanding the function of the input features learned by the model, which is identified by the optimization process as best-performing for a target task. For this reason, Explainable AI (XAI) \citep{DBLP:journals/corr/abs-2006-11371} is a very active research field, due to the importance of gaining insights about the learned models, commonly referred to as \emph{Black-box}. 
 
 Recently, explainability is gathering more and more attention also in Reinforcement Learning settings, where value functions and policies could be parametrized by neural networks, raising the need for explainability in some contexts. A review of the topic can be found in Milani et al.~\citep{10.1145/3616864}, where the authors propose a novel taxonomy beyond Intrinsic vs. Post-hoc techniques, in particular, useful to capture the challenges of Reinforcement Learning. In this learning paradigm, there are other aspects that do not emerge in Supervised Learning, such as the fact that there is a behavior that the agent learns, and not only the action selection (global vs. local). As the authors underlined, the popularity of Post-hoc techniques is connected to the high spread of explainable techniques from Supervised Learning, which are applied to Reinforcement Learning.
 Intrinsic approaches, instead, are based on the exploitation of simple models (e.g., linear models, small decision trees) that can be easily interpreted. An interesting element of this taxonomy is \textit{Learning Process and MDP}, in which we can find work as Maduma et al. \citep{DBLP:conf/aaai/Madumal0SV20}, which uses Structural Causal Models (SCM) to generate an explanation. Differently, Cruz et al. \citep{10.1007/978-3-030-35288-2_6} use episodic memory to link the action selected in a specific state to the probability of task success. Saliency Maps are also a popular technique employed to explain Deep RL agents (Atrey et al.~\citep{DBLP:journals/corr/abs-1912-05743}), since RL algorithms are typically applied to solve games where images are used as input. \\
 \\
 Other approaches to explainability include the exploitation of properties of trees, in popular ensemble algorithms such as Random Forest~\citep{10.1023/A:1010933404324} and XGBoost~\citep{DBLP:journals/corr/ChenG16}, or the use of SHAP~\citep{DBLP:journals/corr/LundbergL17}, which proposes a model-agnostic game-theoretic approach to analyze the output of machine learning models, estimating the contribution of each feature to the output.
 
 Other related works lie in the context of Imitation Learning, where the focus is on learning the behavior of a given policy. In this field, we can find Behavioral Cloning (BC), which simply consists of training a policy via supervised learning on the state-action dataset $(s, a)$ obtained from the expert.
 Some extensions in the direction of Imitation Learning are DAGGER~\citep{DBLP:journals/corr/abs-1011-0686} and VIPER~\citep{NEURIPS2018_e6d8545d}, in which the expert policy is used to determine which action it would take in the state visited by the imitation policy (parametrized as a decision tree), reducing the compounding error~\citep{pmlr-v9-ross10a}. PIRL~\citep{DBLP:journals/corr/abs-1804-02477} is an approach similar to DAGGER, but produces a programmatic policy from a neural network policy.
The main drawback of these approaches is that an expert who can be continuously consulted is necessary, so it is not possible to rely on a fixed dataset of collected trajectories.

Finally, there are also works from other fields, such as Policy Distillation~\citep{DBLP:journals/corr/RusuCGDKPMKH15} and Model Compression~\citep{NIPS2014_b0c355a9}, which are strongly related to our approach. In Policy Distillation, the knowledge of one or more experts is combined in a single, typically shallower, neural network. Here, the objective is different from explainability, but related to efficiency and the improvement that is obtained by combining different agents with respect to the performance of single experts. Xing et al.~\citep{XING2023228} use Policy Distillation to achieve interpretability. Our work lies in the intersection between policy distillation and imitation learning, with the aim of obtaining an interpretable policy from an expert, similar to DAGGER, but with the possibility to rely on a fixed set of trajectories, without requiring continuous interrogation of the expert.

\section{Bound on Policy Improvement}
\label{Sec:Bound}
In this section, we present the main theoretical result of this paper, which is the foundation of the proposed algorithm. The result we discuss is an adaptation of the lower bound presented in a previous work by Pirotta et al. \citep{pmlr-v28-pirotta13} (Section~\ref{Sec:Bound:Subsec:derivation}). We consider the following setting: an existing (complex) policy $\pi_E$ has been previously identified, and we want to learn a policy $\pi_I$ from another (simpler) policy space that is closer to $\pi_E$, such that we obtain an interpretable policy that imitates the original one, focusing on important states, i.e., with comparable value functions. Then, building upon the main theoretical result previously discussed, in Section~\ref{Sec:Bound:Subsec:Algorithm} we propose an algorithm that optimizes the lower bound, providing an algorithmic procedure to learn the interpretable policy of interest.

\subsection{Bound Derivation}
\label{Sec:Bound:Subsec:derivation}
We start by considering Theorem~\ref{Th:Pirotta} presented by Pirotta et al.~\citep{pmlr-v28-pirotta13}, in which the performance difference between two policies $\pi$ and $\pi'$ is lower bounded. In the paper, the authors propose novel lower bounds that are used in the context of policy iteration algorithms, addressing the problem of non-monotonic improvement and oscillation.
In particular, the results show that the advantage is a proper estimate of the performance.


\begin{theorem}[Theorem 3.5 from~\citep{pmlr-v28-pirotta13}]
\label{Th:Pirotta}
For any stationary policies $\pi$ and $\pi'$ and any starting state distribution $\mu$, given any baseline policy $\pi_b$, the difference between the performance of $\pi'$ and one of $\pi$ can be lower bounded as follows:

\begin{equation}
    J_{\mu}^{\pi'} - J_\mu^{\pi} \geq d_{\mu}^{\pi_b^T}A_\pi^{\pi'} - \frac{\gamma}{(1-\gamma)^2} || \Pi^{\pi'} - \Pi^{\pi_b} ||_\infty \frac{\Delta A_{\pi}^{\pi'}}{2}.
\end{equation}
\end{theorem}

The following corollary is a looser but simplified version of the bound. 

\begin{corollary}[Corollary 3.6 from~\citep{pmlr-v28-pirotta13}]
\label{Corr:Pirotta}
\begin{equation}
J_{\mu}^{\pi'} - J_\mu^{\pi} \geq d_{\mu}^{\pi^T}A_\pi^{\pi'} - \frac{\gamma}{(1-\gamma)^2} || \Pi^{\pi'} - \Pi^{\pi} ||_\infty \frac{|| q^{\pi}||_{\infty}}{2}.
\end{equation}
\end{corollary}

In the bound, we can find two main components: the advantage $A_\pi^{\pi'}$, which represents how much we gain in terms of expected return by playing the policy $\pi'$ instead of $\pi$, and a second term, which represents how much the two policies are distant, scaled by some constants that characterize the problem.

Our objective is to train a surrogate policy $\pi_I$ in a different policy space (e.g., linear policy space) w.r.t. an expert policy $\pi_E$, with the purpose of learning a surrogate policy that is as close as possible to the behavior of the given complex policy $\pi_E$.
We can adapt Theorem~\ref{Th:Pirotta} to this purpose, given the generality of the result, obtaining the following result.

\begin{theorem}
\label{Th:1}
For any expert policy $\pi_E$, any surrogate policy $\pi_I$ and any starting state distribution $\mu$, given any baseline policy $\pi_b$, the difference between the performance of $\pi_I$ and one of $\pi_E$ can be bounded as:
\begin{equation}
    J_{\mu}^{\pi_I} - J_\mu^{\pi_E} \geq d_{\mu}^{\pi_b^T}A_{\pi_E}^{\pi_I} - \frac{\gamma}{(1-\gamma)^2} || \Pi^{\pi_I} - \Pi^{\pi_b} ||_\infty \frac{\Delta A_{\pi_E}^{\pi_I}}{2}.
\end{equation}
\end{theorem}
The proof of Theorem \ref{Th:1} follows from the results presented in~\citep{pmlr-v28-pirotta13}.\\
\\
In Theorem~\ref{Th:1}, $\mathbb{A}_{\pi_E, \mu}^{\pi_I}$ represents the disadvantage of following the surrogate policy $\pi_I$ instead of $\pi_E$, since $\pi_E$ is an expert:
\begin{equation*}
    \mathbb{A}_{\pi_E, \mu}^{\pi_I}=\sum_{s \in S} d_\mu^{\pi_{E}}( s)A_{\pi_E}^{\pi_I}(s) = \sum_{s \in S} d_\mu^{\pi_E}(s)\sum_{a \in A} \pi_I(a|s)A^{\pi_E}(s,a).
\end{equation*}

We can rewrite Corollary~\ref{Corr:Pirotta}, obtaining the following result, which is a looser but simplified bound of the previous theorem.
\begin{corollary}
\label{Co:1}
\begin{equation}
    J_{\mu}^{\pi_I} - J_\mu^{\pi_E} \geq d_{\mu}^{\pi_E^T}A_{\pi_E}^{\pi_I} - \frac{\gamma}{(1-\gamma)^2} || \Pi^{\pi_I} - \Pi^{\pi_E} ||_\infty \frac{|| q^{\pi_E}||_{\infty}}{2}.
\end{equation}
\end{corollary}

Considering as upper bound of the quantity $|| q^{\pi_E}||_{\infty} $ its maximum value $\frac{1}{1-\gamma}$, we can obtain the following equation (similar to Equation 1 from~\citep{pmlr-v28-pirotta13}):
\begin{equation}
\label{optimizing}
    J_{\mu}^{\pi_I} - J_\mu^{\pi_E} \geq \mathbb{A}_{\pi_E, \mu}^{\pi_I} - \frac{\gamma}{2(1-\gamma)^3} || \Pi^{\pi_I} - \Pi^{\pi_E}||_{\infty}^2.
\end{equation}
 
 Both terms in Equation~\ref{optimizing} are crucial for our objective, and they capture two different aspects that we want to optimize simultaneously. Indeed, since we are interested in explaining an expert policy $\pi_E$, we need a surrogate policy that imitates the expert one by associating similar probabilities to state-action pairs, and with a similar performance, measured through the advantage function, which suggests focusing on the imitation of state-action pairs with high reward.
 
 To optimize the bound in Equation~\ref{optimizing}, we finally focus on the gradient of both terms, making their computation explicit.

\begin{definition}
The expected advantage of $\pi_\theta$ over $\pi$ is defined as:
\begin{equation}
    J(\theta) = \mathbb{E}_{s \sim d^{\pi_E},  a \sim \pi_\theta} [A^\pi(s,a)].
\end{equation}

\end{definition}

\begin{lemma}
\label{Lemma:1}
Given a parametric policy $\pi_{\theta}$, the gradient of the expected advantage of $\pi_\theta$ over $\pi$ is:
\begin{align}
    & \nabla_\theta J(\theta) = \mathbb{E}_{s \sim d^{\pi_E}, a \sim \pi_\theta} \big{[} {\nabla_\theta} \text{log} \pi_{\theta}(a|s) A^{\pi_E}(s,a) \big{]}.
\end{align}
\end{lemma}

\begin{proof}
    \begin{align*}
    \nabla_\theta J(\theta) &= \nabla_\theta \big{[} \sum_{s \in S} d_\mu^{\pi_E}(s)\sum_{a \in A} \pi_{\theta}(a|s)A^{\pi_{E}}(s,a) \big{]} \\ 
    &=\nabla_\theta \big{[} \sum_{s \in S}\sum_{a \in A} \pi_{\theta}(a|s)A^{\pi_E}(s,a) \big{]} \\ 
    &=\sum_{s \in S} \sum_{a \in A} \nabla_\theta \pi_{\theta}(a|s)A^{\pi_E}(s,a)\\
    &=\sum_{s \in S} \sum_{a \in A}  \frac{\pi_{\theta}(a|s)}{\pi_{\theta}(a|s)}\nabla_\theta \pi_{\theta}(a|s)A^{\pi_E}(s,a) \\
     &=\sum_{s \in S} \sum_{a \in A} \pi_{\theta}(a|s)\nabla_\theta \text{log}(\pi_{\theta}(a|s))A^{\pi_E}(s,a) \label{eq:adv}\\
     &=\mathbb{E}_{s \sim d^{\pi_E}, a \sim \pi_\theta} \big{[} {\nabla_\theta} \text{log} \pi_{\theta}(a|s) A^{\pi_E}(s,a) \big{]}.
\end{align*}
\end{proof}

For the second term, we substitute the infinite norm $L_\infty$ with the sum of squares of the differences between the two policies in each state-action pair, since we need the second component of the bound to enforce a uniform imitation, leaving the focus on specific state-action pairs to the first term. We therefore consider the following loss measure for the second term, which is an upper bound of the $L_\infty$ norm that appears in the original lower bound. 

\begin{definition}
Given the policy of an expert $\pi_E$ and a policy $\pi_\theta$, the following loss measure is an upper bound of the (squared) $L_\infty$ norm:
    \begin{equation}
    L(\theta) = \sum_{s \in S} \sum_{a \in A} |\pi_\theta(a|s) - \pi_E(a|s)|^2.
\end{equation}
\end{definition}

\begin{lemma}
\label{Lemma:2}
Given the policy of an expert $\pi_E$ and a policy $\pi_\theta$, the gradient of the $L_{\theta}$ loss is: 

\begin{equation}
     \nabla_\theta L(\theta) = \\
      \mathbb{E}_{\substack{s \sim d^{\pi_E},\\ a \sim \pi_\theta}} \big{[} {2\nabla_\theta} \text{log} \pi_{\theta}(a|s) (\pi_\theta (a|s) - \pi_E (a|s)) \big{]}.
\end{equation}
\end{lemma}

\begin{proof}
    \begin{align*}
    \nabla_\theta L(\theta) &= \nabla_\theta\sum_{s \in S} \sum_{a \in A} |\pi_\theta(a|s) - \pi_E(a|s)|^2 \\
    &= \nabla_\theta \sum_{s \in S} \sum_{a \in A} (\pi_\theta(a|s) - \pi_E(a|s))^2 \\
     &= 2 \sum_{s \in S} \sum_{a \in A} (\pi_\theta(a|s) - \pi_E(a|s)) \nabla_\theta \pi_\theta(a|s) \\
     &= 2 \sum_{s \in S} \sum_{a \in A} (\pi_\theta(a|s) - \pi_E(a|s)) \pi_\theta(a|s) \nabla_\theta \text{log}(\pi_\theta(a|s)) \\
     & =\mathbb{E}_{s \sim d^{\pi_E}, a \sim \pi_\theta} \big{[} {2\nabla_\theta} \text{log} \pi_{\theta}(a|s) (\pi_\theta (a|s) - \pi_E (a|s)) \big{]}.
\end{align*}
\end{proof}

In conclusion, in this section we introduced Theorem~\ref{Th:1}, which states that the performance difference between an expert $\pi_E$ and a policy $\pi_I$ can be lower bounded by two terms, one related to the advantage of $\pi_I$ over $\pi_E$, and a term related to their similarity. Then, in Lemma~\ref{Lemma:1} and  Lemma~\ref{Lemma:2} the gradients of the two terms have been computed, and they will be exploited to minimize the proposed performance difference bound. Lemma~\ref{Lemma:1} is related to the disadvantage of playing the policy $\pi_I$ instead of the expert, and we want to minimize it to obtain the best-performing policy. Indeed, given the constraint of using a different (simplified) policy space, it is not guaranteed that it is possible to reach a disadvantage of 0. Lemma~\ref{Lemma:2} instead is related to the discrepancy between the two policies, which we want to keep as small as possible, for the purpose of explainability. A relevant example is the linear case, where we don't want to identify the best linear policy, but we seek a linear policy that is as similar as possible to the expert one, with a particular interest in mimicking its performance.

\subsection{Imitation Algorithm}
Using the bound presented in the previous section, it is possible to introduce an algorithmic procedure that exploits this quantity to optimize the surrogate policy. The pseudo-code of this algorithm, which we called Explainable Policy distillation with Augmented Behavioral Cloning (EXPLAIN), is reported in Algorithm~\ref{alg:cap}.
In particular, given an expert policy $\pi_E$ that could be the optimal policy $\pi^*$ or an expert well-performing policy in general, we collect a number $\tau$ of trajectories composed of tuples $(s, a, Q_{a_0}, ..., Q_{a_n})$, where $Q_{a_i}$ represents the action value function $Q(s, a_i)$ that is used to estimate the advantage function $A(s,a)$. In general, it is also possible to rely on algorithms that directly estimate the advantage.
Then, for each optimization step, the gradient of the two components of the bound presented in Theorem~\ref{Th:1} is computed, respectively accounting for the (dis)advantage with respect to the expert policy $\pi_E$ (Lemma~\ref{Lemma:1}) and for the distance between the two policies $\pi_\theta$ and $\pi_E$ (Lemma~\ref{Lemma:2}).

\label{Sec:Bound:Subsec:Algorithm}
\begin{algorithm}
\caption{EXPLAIN}\label{alg:cap}
\begin{algorithmic}
\Require $\pi_E$ (Expert policy), $N \geq 0$ (Number of Iterations), $\lambda$ (Learning Rate), $\eta$ (BC penalization)
\State collect trajectories $\tau$ using $\pi_E$
\State initialize  $\pi_{\theta}^{LR}$, Interpretable Policy (Softmax Regression)
\While{$N \neq 0$}
    \State $\nabla J(\theta) = \mathbf{E}_{\substack{s\sim \tau}}[\nabla_\theta \text{log} \pi_{\theta}^{LR}(a | s) A^{\pi_E}(s,a)$]
    \State $\nabla L(\theta)=\mathbf{E}_{\substack{s \sim \tau}} \big{[} {\nabla_\theta} \text{log} \pi_{\theta}(a|s) (\pi_\theta (a|s) - \pi_E (a|s)) \big{]}$
    
    \State $\theta \gets \theta + \lambda(\nabla J_A(\theta) - \eta \nabla L(\theta))$
    \State $N \gets N - 1$
\EndWhile
\end{algorithmic}
\end{algorithm}

Finally, $N, \lambda$, and $\eta$ are the hyperparameters of the algorithm, where $\lambda$ is the initial learning rate for Adam~\citep{Kingma2014AdamAM}, $\eta$ is a scaling factor of the Behavioral Cloning component, and $N$ is the number of iterations.
As can be observed, we need to sample only once a dataset from the expert $\pi_E$, and then we can iterate over it to optimize the Softmax Regression. As mentioned, the presence of the advantage allows us to include information about the expected return of a given action, thus we expect that the policy obtained from this process is able to focus more on actions with higher importance.

\section{Experiments}
\label{Sec:Experiments}
In this Section, we empirically validate the proposed algorithm on synthetic and real-world scenarios\footnote{Code available at \url{https://github.com/giovannidispoto/distillation-interpretable-policy}}.
Firstly, in Section~\ref{subsec:classic}, we trained an expert using DQN~\citep{DBLP:journals/corr/MnihKSGAWR13} on classic control tasks as a modified version of Inverted Pendulum and a discretized version of Mountain Car Continuous. Both environments present a large action space when discretized, which makes behavioral cloning harder. In particular, in the Inverted Pendulum, small errors can quickly compound and lead to failure. In Section~\ref{subsec:explain_trading}, we exploit the proposed algorithm to gain insight about an agent trained on a financial trading task. Since in real-world applications not all the features are equally important, we empirically demonstrate that our algorithm can effectively identify the relevant ones in this setting.

\subsection{Imitation Learning in Classic Control Environments}
\label{subsec:classic}
As a first empirical assessment, we trained DQN~\citep{DBLP:journals/corr/MnihKSGAWR13} using Stable-Baselines 3~\citep{stable-baselines3} on a discretized version of Inverted Pendulum from Gymnasium~\citep{towers2024gymnasium} with 100 actions. 
The Inverted Pendulum task consists of balancing a pole that is attached to a cart by applying forces to the cart to the left and to the right. The action represents the amount of force to apply to the cart and the relative sign (left or right). This environment is a relevant task for our approach since it is a relatively simple environment accompanied by a known optimal policy, but small errors quickly compound and lead to failure.
In particular, we opted for a sine transformation of the action, in order to emulate an environment in which different actions have the same impact, simply by discretizing the interval $[0, 4\pi]$ in the number of actions and then by applying the action $a_{env} = 3sin(a_{pred})$.
Before training the surrogate policy using Algorithm~\ref{alg:cap}, the features of the state are standardized. 

From the results, as can be seen in Figure~\ref{fig:Invertedpendulum},  we can notice that the Behavioral Cloning component, i.e., the second part of the lower bound discussed in the previous section, is able to extract a linear policy with performance close to the expert. Then, by including the first component related to the advantage, we are able to reach a higher mean performance with smaller variance, due to the fact that the advantage function can point out that different actions have the same expected return. In the figure, different dataset sizes are reported, i.e., an increasing number of expert trajectories is collected, showing increasing performances overall. Specifically, we can notice from the plots that the performance of the Behavioral Cloning component increases when more trajectories are collected from the expert policy.

\begin{figure}[t]
\centering
\begin{subfigure}{=0.3\textwidth}    
\centering
            \includegraphics[width=1.1\linewidth]{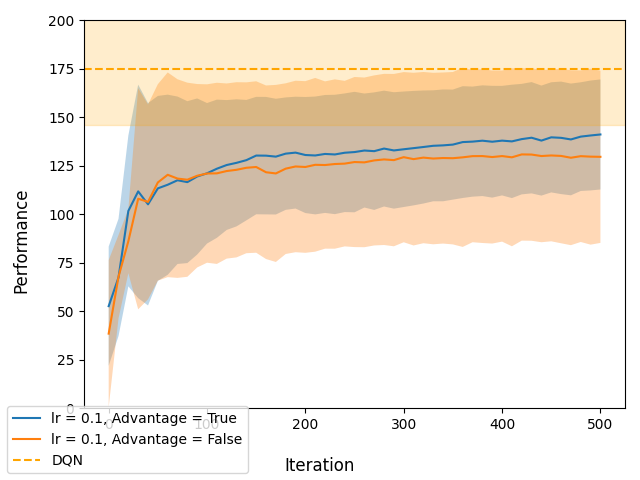}
            \caption{Dataset size 3, $\eta$ = 0.01}
            \label{fig:ds3_ip}
\end{subfigure}%
\vspace{0.5cm}
\hfill
\begin{subfigure}{=0.3\textwidth}  
\centering
\includegraphics[width=1.1\linewidth]{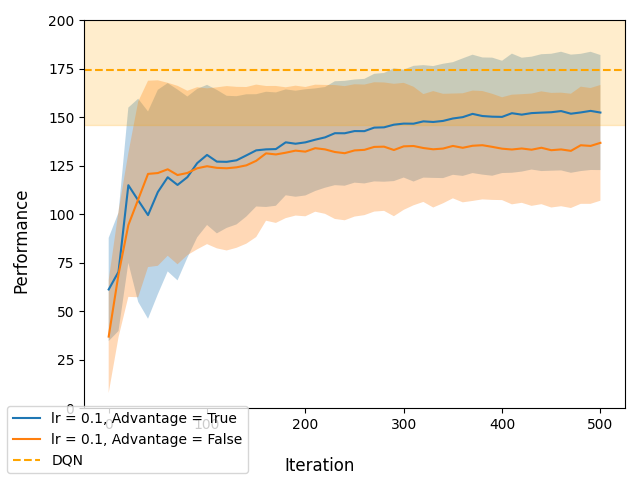}
            \caption{Dataset size 5, $\eta$ = 0.01}
            \label{fig:ds5_ip}
\end{subfigure}%
\vspace{0.5cm}
\hfill
\begin{subfigure}{=0.3\textwidth}   
\centering
            \includegraphics[width=1.1\linewidth]{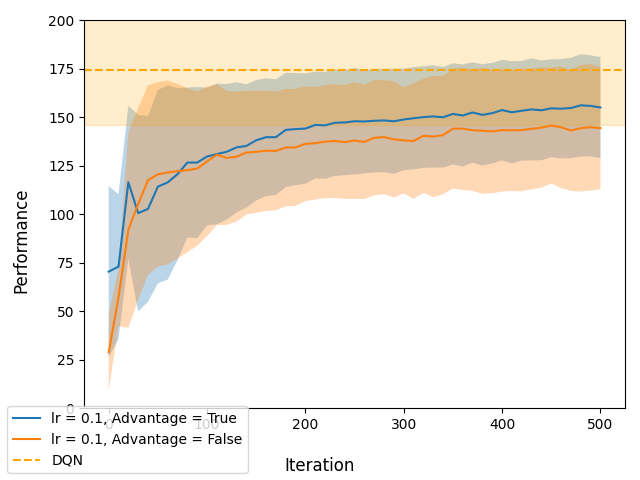}
            \caption{Dataset size 10, $\eta$ = 0.01}
            \label{fig:ds10_ip}
    \end{subfigure}%
    \vspace{0.5cm}
\caption{Performances of the expert (DQN) and of the linear policy on the modified Inverted Pendulum (=100 actions) task with different numbers of trajectories collected from the expert (Dataset size). The blue curve is obtained by optimizing Adv+BC, and the orange curve is only BC.
Mean and Standard Deviation over 6 seeds.}
\vspace{0.5cm}
\label{fig:Invertedpendulum}
\end{figure}

As a second experiment on classic control environments, we trained DQN on a discretized version of Mountain Car Continuous (=50 actions), in which a car is in the middle of a valley and needs to escape from it. The actions represent the directional force to apply to the car. To escape the valley, the policy needs to learn to accelerate to the left and to the right in order to build momentum. 
As it is possible to see in Figure~\ref{fig:MountainCar}, also in this case we are able to learn a linear policy from the expert policy, with comparable performances. In Figure~\ref{fig:mountain_car_policy}, the weights associated with each feature are further reported, underlying the clear interpretability of the learned surrogate linear policy. Unlike the previous setting, small errors in this environment have a limited impact, thus the Behavioral Cloning component of our algorithm is sufficient in this case to learn an interpretable linear policy from the expert one.

\begin{figure}[t]
\centering
\begin{subfigure}[b]{=0.4\textwidth}    
\centering
\includegraphics[width=\linewidth]{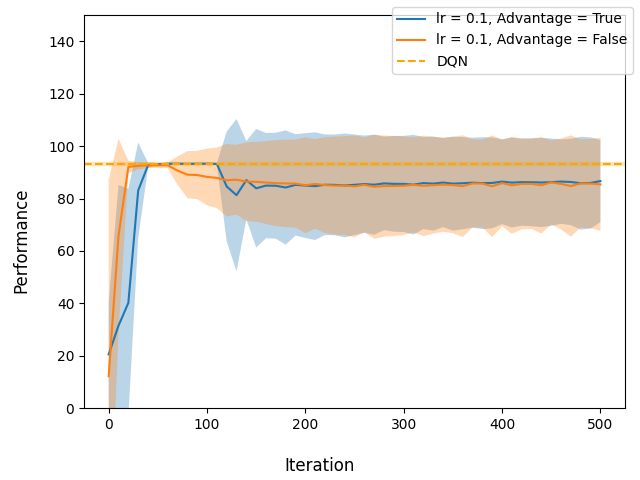}
            \caption{Dataset size 5, $\eta$ = 1e-3}
            \label{fig:ds5_mc}
\end{subfigure}%
\vspace{0.5cm}
\hfill
\begin{subfigure}[b]{=0.4\textwidth}  
\centering
\includegraphics[width=\linewidth]{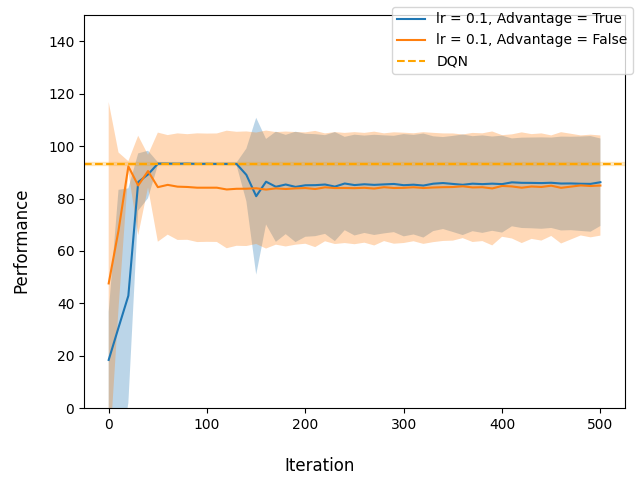}
            \caption{Dataset size 10, $\eta$ = 1e-3}
            \label{fig:ds10_mc}
\end{subfigure}%
    \vspace{0.5cm}
\caption{Performances of the expert (DQN) and of the linear policy on discretized Mountain Car Continuous (=50 actions) with different numbers of trajectories collected from the expert (Dataset size). The blue curve is obtained by optimizing Adv+BC, and the orange curve is only BC. Mean and Standard Deviation over 6 seeds. }
\vspace{0.5cm}
\label{fig:MountainCar}
\end{figure}

\begin{figure}[hbtp]
\centering
\includegraphics[width=0.34\textwidth]{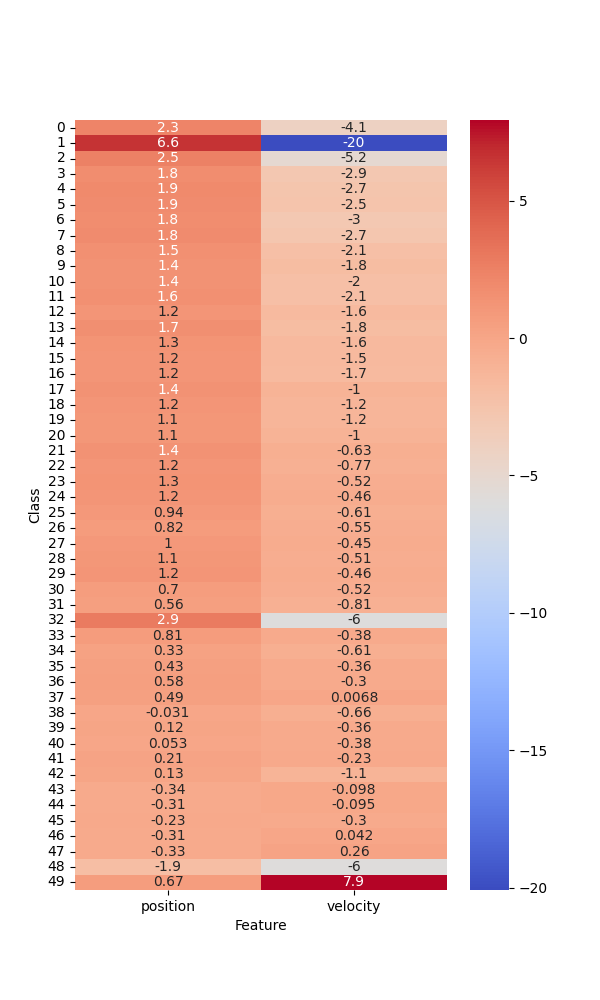}
\caption{Weights of the softmax regression related to each feature in Mountain Car and to each class. High weights on the position in the first classes can be associated with the behavior of pushing the car to the left in order to build momentum. Regarding the last action, it is likely associated with pushing the car to the right after gaining momentum. }

\label{fig:mountain_car_policy}
\end{figure}

\subsection{Explaining Financial Trading Agents}
\label{subsec:explain_trading}

To evaluate the explainability of the proposed algorithm in a real-world challenging scenario, where it is crucial to understand the reasoning behind the decisions of the agent, while limiting the effect on the performance, we consider a real-world task such as financial trading. Typically, these environments are populated by a large number of features, unlike toy-control tasks, but not all of them are equally important. In this context, we trained FQI (Fitted Q-iteration, ~\citep{JMLR:v6:ernst05a}) with XGBoost~\citep{DBLP:journals/corr/ChenG16} as a regressor (one for each action), to trade a financial instrument, using data from a Limit Order Book (LOB)~\citep{gould2013limitorderbooks}. In this environment, the state is made up of different features such as temporal information, prices, and trading volumes, and the agent, at each trading step, decides to go long (buy a single unit of the asset), flat (no action), or short (short sell a single unit of the asset). The reward function takes into account only the price variation $R_{t+1} = a_t (price_{t+1} - price_{t})$ (e.g., if the agent short-sells the asset and the price goes from 101€ to 100€, the agent gains 1€). In this setting, due to the absence of transaction costs, the flat action is not shown in the various plots (even if it can be selected by the agent). For the same reason, FQI can be trained considering only the first iteration, due to the limited or no additional benefit in subsequent iterations.
Since XGBoost has been selected as a regressor, we can easily get information about the importance of the features, as reported in Figure~\ref{fig:Explainability_trading_regressor}. From this result, it is possible to notice that most of the features are of high importance, with some peaks and some others having negligible contributions. 
\begin{figure}[hbtp]
\centering
\includegraphics[width=1\linewidth]{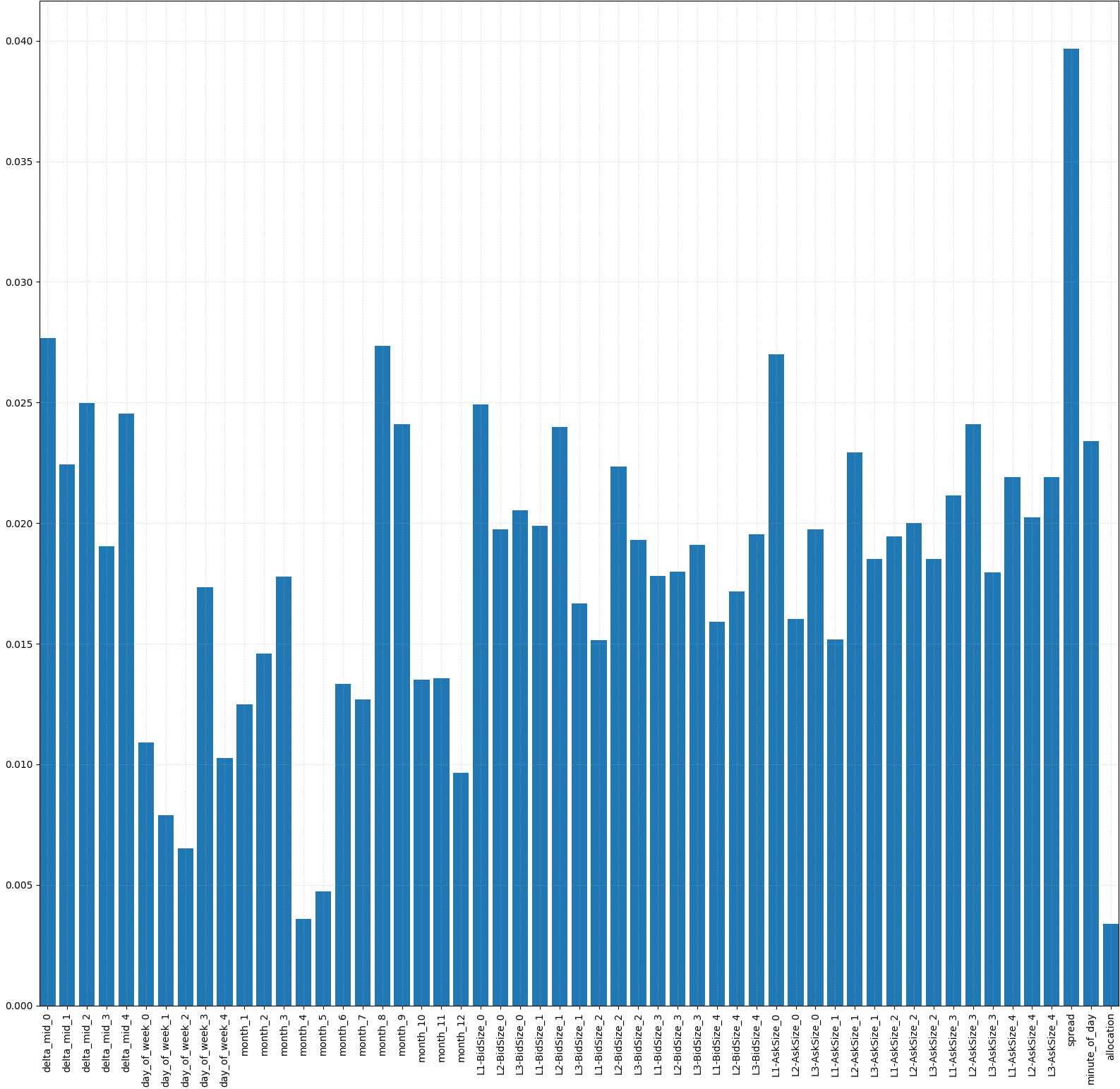}
\vspace{0.5cm}
\caption{Features Importance of the XGBoost regressors related to short action.}
\vspace{0.5cm}
\label{fig:Explainability_trading_regressor}
\end{figure}

In this setting, we trained Algorithm~\ref{alg:cap}, resulting in a surrogate linear policy that can be explicitly interpreted by a domain expert. Firstly, as can be seen from Figure~\ref{fig:fqi_trading}, the raw performances are satisfactory compared w.r.t. the expert policy.

\begin{figure}[hbtp]
\centering
\includegraphics[width=0.45\textwidth]{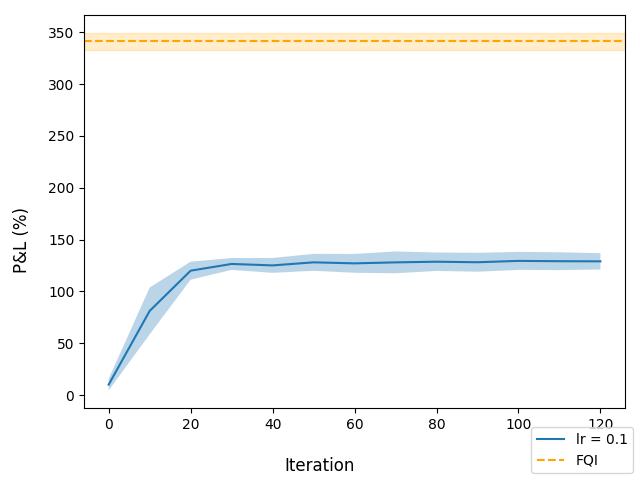}
\caption{Performances of the expert (FQI) and of the linear policy on the Financial Trading task. The P\&L (Profit and Loss) is calculated as a percentage of the initial capital. Mean and standard deviation over 6 seeds.}
\label{fig:fqi_trading}
\vspace{0.5cm}
\end{figure}

Additionaly, by plotting the evolution of the coefficients of the softmax regression (Figure~\ref{fig:weights_evolution_softmax}), we can observe that the agent learned a simple imbalance rule (Equation~\ref{eq:ImbRule}) on the volumes on the first level of the order book, that is profitable in absence of transaction costs:

\begin{figure*}[hbtp]
\centering
\begin{subfigure}[b]{=0.4\textwidth}    
\centering
            \includegraphics[width=0.7\linewidth]{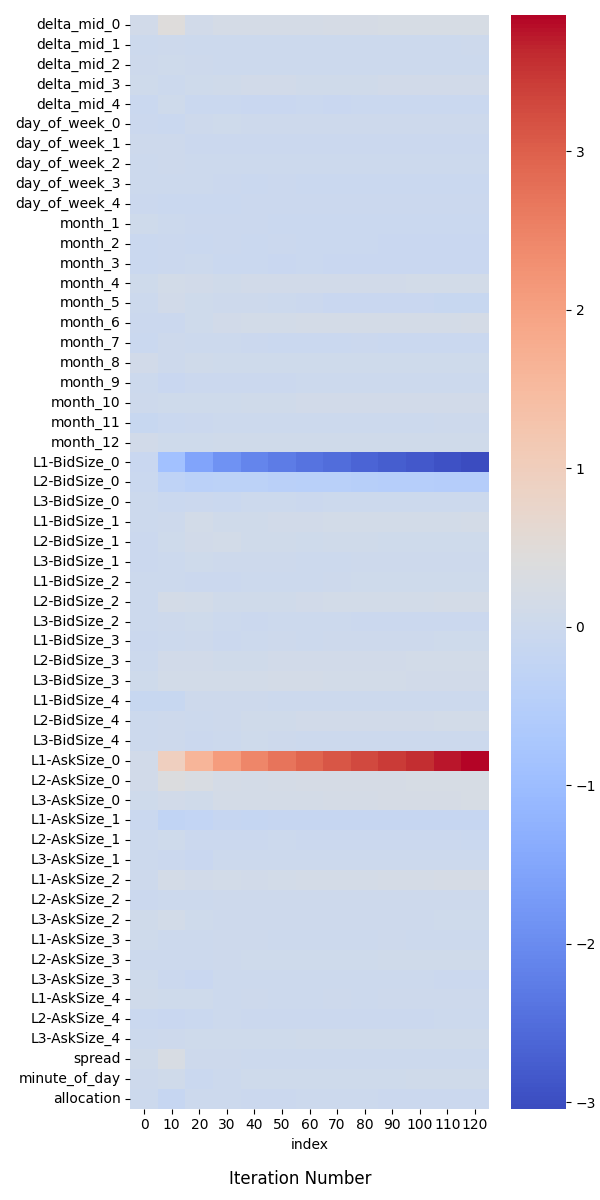}
            \caption{Short}
            \label{fig:short_weights}
\end{subfigure}%
\hfill
\begin{subfigure}[b]{=0.4\textwidth}  
\centering
\includegraphics[width=0.7\linewidth]{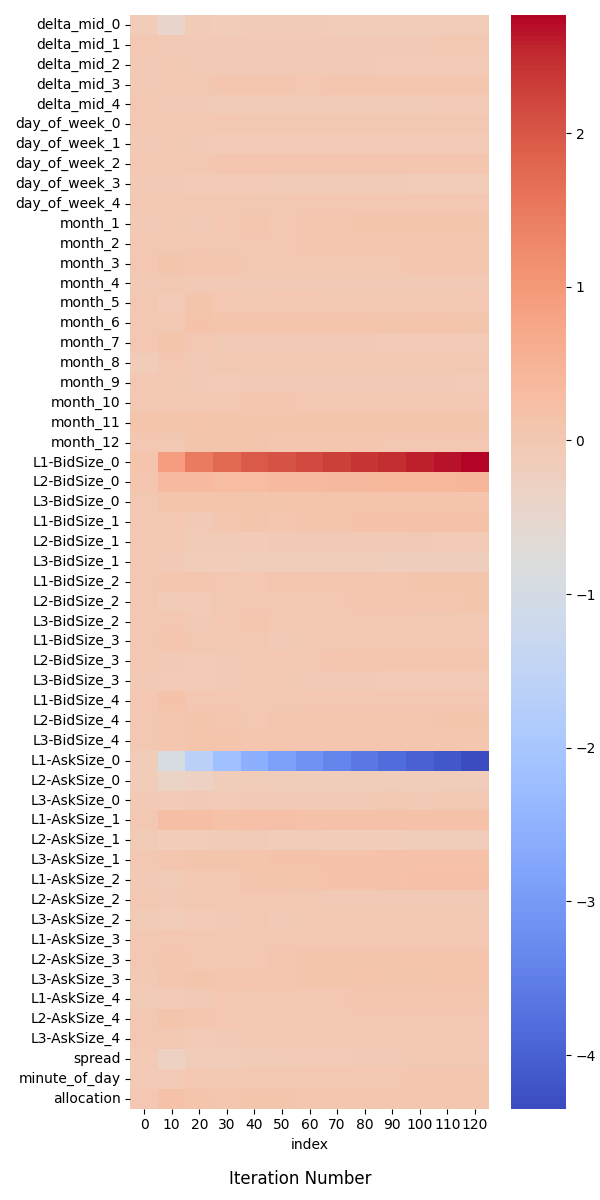}
            \caption{Long}
            \label{fig:long_weights}
\end{subfigure}%
\vspace{0.3cm}
\caption{Weights evolution of the softmax regression related to each feature. In Figure~\ref{fig:short_weights} are visible the weights related to the short action, and in Figure~\ref{fig:long_weights}, the weights of the long action. The importance of the two features is close in absolute value, but different in sign, reflecting the opposite effect of the two actions. Flat action is not reported since it is not relevant in a setting without transaction costs.}
\vspace{0.2cm}
\label{fig:weights_evolution_softmax}
\end{figure*}

\begin{align}
    \label{eq:ImbRule}
    \pi(s) = \begin{cases} 
    \text{long},  &\text{if L1-BidSize} > \text{L1-AskSize}    \\[1ex]
    \text{short},  &\text{if L1-BidSize} < \text{L1-AskSize} \\[1ex]
    \text{flat} & \text{otherwise}
  \end{cases}
\end{align}
The imbalance rule is intuitive, since a large bid size indicates that the traders want to buy the assets, with a subsequent price increase. On the other hand, a large ask size suggests the opposite, indicating a negative market sentiment. In the end, with our algorithm, we can obtain a linear policy that is able to capture most of the information and a discrete profit, suggesting a clear trading strategy to market experts in this context that would not be possible with the original complex policy.

\section{Discussion and Conclusions}
\label{Sec:Conclusions}

In this paper, we introduced a novel approach designed to extract an interpretable policy from a complex expert policy trained using Deep RL techniques. Motivated by the need for transparency in safety-critical domains, our approach aims to retain the essential behavior of an expert agent while producing a policy that is simpler and easier to interpret.

The central idea is to move beyond standard Behavioral Cloning, which treats all imitation targets uniformly. Instead, our method introduces an advantage-aware mechanism that prioritizes fidelity in high-stakes decisions—those where the choice of action significantly influences long-term outcomes. In contrast, when the expert’s actions have low relative advantage, the algorithm favors simplicity, allowing the surrogate model to default to more interpretable behavior. This selective imitation enables a better trade-off between transparency and performance.

We formalized this intuition by adapting a theoretical bound from Pirotta et al.~\citep{pmlr-v28-pirotta13}, guiding the surrogate training objective. Our empirical results confirmed the practical benefits of this approach. In a discrete environment with a large action space and compounding errors, our method outperformed standard Behavioral Cloning both in terms of average return and robustness. Furthermore, in a real-world financial trading task, we demonstrated that the resulting linear policy not only mirrored the expert's critical decisions, but also uncovered clear behavioral patterns that were not identified by conventional feature importance analyses.

A key limitation of our current approach lies in the expressiveness of linear models, which may struggle to capture more intricate dynamics. Future work could explore more flexible yet interpretable policy classes, such as generalized additive models or decision rules. Additionally, extending the method to continuous action spaces through regression-based surrogates would further increase its applicability.

\begin{ack}
This paper is supported by FAIR project, funded by the NextGen-
erationEU program within the PNRR-PE-AI scheme
(M4C2, Investment 1.3, Line on Artificial Intelligence) and by PNRR (M4C2 Investment 3.3, Innovative PhD programs - CUP D43C23002220008).
\end{ack}


\newpage
\bibliography{mybibfile}

\begin{thebibliography}{33}
\providecommand{\natexlab}[1]{#1}
\providecommand{\url}[1]{\texttt{#1}}
\expandafter\ifx\csname urlstyle\endcsname\relax
  \providecommand{\doi}[1]{doi: #1}\else
  \providecommand{\doi}{doi: \begingroup \urlstyle{rm}\Url}\fi

\bibitem[Atrey et~al.(2019)Atrey, Clary, and
  Jensen]{DBLP:journals/corr/abs-1912-05743}
A.~Atrey, K.~Clary, and D.~D. Jensen.
\newblock Exploratory not explanatory: Counterfactual analysis of saliency maps
  for deep {RL}.
\newblock \emph{CoRR}, abs/1912.05743, 2019.
\newblock URL \url{http://arxiv.org/abs/1912.05743}.

\bibitem[Ba and Caruana(2014)]{NIPS2014_b0c355a9}
L.~J. Ba and R.~Caruana.
\newblock Do deep nets really need to be deep?
\newblock In Z.~Ghahramani, M.~Welling, C.~Cortes, N.~Lawrence, and
  K.~Weinberger, editors, \emph{Advances in Neural Information Processing
  Systems}, volume~27. Curran Associates, Inc., 2014.
\newblock URL
  \url{https://proceedings.neurips.cc/paper_files/paper/2014/file/b0c355a9dedccb50e5537e8f2e3f0810-Paper.pdf}.

\bibitem[Bastani et~al.(2018)Bastani, Pu, and
  Solar-Lezama]{NEURIPS2018_e6d8545d}
O.~Bastani, Y.~Pu, and A.~Solar-Lezama.
\newblock Verifiable reinforcement learning via policy extraction.
\newblock In S.~Bengio, H.~Wallach, H.~Larochelle, K.~Grauman, N.~Cesa-Bianchi,
  and R.~Garnett, editors, \emph{Advances in Neural Information Processing
  Systems}, volume~31. Curran Associates, Inc., 2018.
\newblock URL
  \url{https://proceedings.neurips.cc/paper_files/paper/2018/file/e6d8545daa42d5ced125a4bf747b3688-Paper.pdf}.

\bibitem[Breiman(2001)]{10.1023/A:1010933404324}
L.~Breiman.
\newblock Random forests.
\newblock \emph{Mach. Learn.}, 45\penalty0 (1):\penalty0 5–32, Oct. 2001.
\newblock ISSN 0885-6125.
\newblock \doi{10.1023/A:1010933404324}.
\newblock URL \url{https://doi.org/10.1023/A:1010933404324}.

\bibitem[Briola et~al.(2021)Briola, Turiel, Marcaccioli, and
  Aste]{DBLP:journals/corr/abs-2101-07107}
A.~Briola, J.~D. Turiel, R.~Marcaccioli, and T.~Aste.
\newblock Deep reinforcement learning for active high frequency trading.
\newblock \emph{CoRR}, abs/2101.07107, 2021.
\newblock URL \url{https://arxiv.org/abs/2101.07107}.

\bibitem[Chen and Guestrin(2016)]{DBLP:journals/corr/ChenG16}
T.~Chen and C.~Guestrin.
\newblock Xgboost: {A} scalable tree boosting system.
\newblock \emph{CoRR}, abs/1603.02754, 2016.
\newblock URL \url{http://arxiv.org/abs/1603.02754}.

\bibitem[Cruz et~al.(2019)Cruz, Dazeley, and
  Vamplew]{10.1007/978-3-030-35288-2_6}
F.~Cruz, R.~Dazeley, and P.~Vamplew.
\newblock Memory-based explainable reinforcement learning.
\newblock In J.~Liu and J.~Bailey, editors, \emph{AI 2019: Advances in
  Artificial Intelligence}, pages 66--77, Cham, 2019. Springer International
  Publishing.
\newblock ISBN 978-3-030-35288-2.

\bibitem[Das and Rad(2020)]{DBLP:journals/corr/abs-2006-11371}
A.~Das and P.~Rad.
\newblock Opportunities and challenges in explainable artificial intelligence
  {(XAI):} {A} survey.
\newblock \emph{CoRR}, abs/2006.11371, 2020.
\newblock URL \url{https://arxiv.org/abs/2006.11371}.

\bibitem[Ernst et~al.(2005)Ernst, Geurts, and Wehenkel]{JMLR:v6:ernst05a}
D.~Ernst, P.~Geurts, and L.~Wehenkel.
\newblock Tree-based batch mode reinforcement learning.
\newblock \emph{Journal of Machine Learning Research}, 6\penalty0
  (18):\penalty0 503--556, 2005.
\newblock URL \url{http://jmlr.org/papers/v6/ernst05a.html}.

\bibitem[Gould et~al.(2013)Gould, Porter, Williams, McDonald, Fenn, and
  Howison]{gould2013limitorderbooks}
M.~D. Gould, M.~A. Porter, S.~Williams, M.~McDonald, D.~J. Fenn, and S.~D.
  Howison.
\newblock Limit order books, 2013.
\newblock URL \url{https://arxiv.org/abs/1012.0349}.

\bibitem[Haarnoja et~al.(2018)Haarnoja, Zhou, Abbeel, and
  Levine]{DBLP:journals/corr/abs-1801-01290}
T.~Haarnoja, A.~Zhou, P.~Abbeel, and S.~Levine.
\newblock Soft actor-critic: Off-policy maximum entropy deep reinforcement
  learning with a stochastic actor.
\newblock \emph{CoRR}, abs/1801.01290, 2018.
\newblock URL \url{http://arxiv.org/abs/1801.01290}.

\bibitem[Ho and Ermon(2016)]{DBLP:journals/corr/HoE16}
J.~Ho and S.~Ermon.
\newblock Generative adversarial imitation learning.
\newblock \emph{CoRR}, abs/1606.03476, 2016.
\newblock URL \url{http://arxiv.org/abs/1606.03476}.

\bibitem[Kaplan et~al.(2020)Kaplan, McCandlish, Henighan, Brown, Chess, Child,
  Gray, Radford, Wu, and Amodei]{DBLP:journals/corr/abs-2001-08361}
J.~Kaplan, S.~McCandlish, T.~Henighan, T.~B. Brown, B.~Chess, R.~Child,
  S.~Gray, A.~Radford, J.~Wu, and D.~Amodei.
\newblock Scaling laws for neural language models.
\newblock \emph{CoRR}, abs/2001.08361, 2020.
\newblock URL \url{https://arxiv.org/abs/2001.08361}.

\bibitem[Kingma and Ba(2014)]{Kingma2014AdamAM}
D.~P. Kingma and J.~Ba.
\newblock Adam: A method for stochastic optimization.
\newblock \emph{CoRR}, abs/1412.6980, 2014.
\newblock URL \url{https://api.semanticscholar.org/CorpusID:6628106}.

\bibitem[Likmeta et~al.(2020)Likmeta, Metelli, Tirinzoni, Giol, Restelli, and
  Romano]{LIKMETA2020103568}
A.~Likmeta, A.~M. Metelli, A.~Tirinzoni, R.~Giol, M.~Restelli, and D.~Romano.
\newblock Combining reinforcement learning with rule-based controllers for
  transparent and general decision-making in autonomous driving.
\newblock \emph{Robotics and Autonomous Systems}, 131:\penalty0 103568, 2020.
\newblock ISSN 0921-8890.
\newblock \doi{https://doi.org/10.1016/j.robot.2020.103568}.
\newblock URL
  \url{https://www.sciencedirect.com/science/article/pii/S0921889020304085}.

\bibitem[Liu et~al.(2018)Liu, Schulte, Zhu, and
  Li]{DBLP:journals/corr/abs-1807-05887}
G.~Liu, O.~Schulte, W.~Zhu, and Q.~Li.
\newblock Toward interpretable deep reinforcement learning with linear model
  u-trees.
\newblock \emph{CoRR}, abs/1807.05887, 2018.
\newblock URL \url{http://arxiv.org/abs/1807.05887}.

\bibitem[Lundberg and Lee(2017)]{DBLP:journals/corr/LundbergL17}
S.~M. Lundberg and S.~Lee.
\newblock A unified approach to interpreting model predictions.
\newblock \emph{CoRR}, abs/1705.07874, 2017.
\newblock URL \url{http://arxiv.org/abs/1705.07874}.

\bibitem[Madumal et~al.(2020)Madumal, Miller, Sonenberg, and
  Vetere]{DBLP:conf/aaai/Madumal0SV20}
P.~Madumal, T.~Miller, L.~Sonenberg, and F.~Vetere.
\newblock Explainable reinforcement learning through a causal lens.
\newblock In \emph{The Thirty-Fourth {AAAI} Conference on Artificial
  Intelligence, {AAAI} 2020, The Thirty-Second Innovative Applications of
  Artificial Intelligence Conference, {IAAI} 2020, The Tenth {AAAI} Symposium
  on Educational Advances in Artificial Intelligence, {EAAI} 2020, New York,
  NY, USA, February 7-12, 2020}, pages 2493--2500. {AAAI} Press, 2020.
\newblock URL \url{https://aaai.org/ojs/index.php/AAAI/article/view/5631}.

\bibitem[Milani et~al.(2024)Milani, Topin, Veloso, and Fang]{10.1145/3616864}
S.~Milani, N.~Topin, M.~Veloso, and F.~Fang.
\newblock Explainable reinforcement learning: A survey and comparative review.
\newblock \emph{ACM Comput. Surv.}, 56\penalty0 (7), Apr. 2024.
\newblock ISSN 0360-0300.
\newblock \doi{10.1145/3616864}.
\newblock URL \url{https://doi.org/10.1145/3616864}.

\bibitem[Mnih et~al.(2013)Mnih, Kavukcuoglu, Silver, Graves, Antonoglou,
  Wierstra, and Riedmiller]{DBLP:journals/corr/MnihKSGAWR13}
V.~Mnih, K.~Kavukcuoglu, D.~Silver, A.~Graves, I.~Antonoglou, D.~Wierstra, and
  M.~A. Riedmiller.
\newblock Playing atari with deep reinforcement learning.
\newblock \emph{CoRR}, abs/1312.5602, 2013.
\newblock URL \url{http://arxiv.org/abs/1312.5602}.

\bibitem[Pirotta et~al.(2013)Pirotta, Restelli, Pecorino, and
  Calandriello]{pmlr-v28-pirotta13}
M.~Pirotta, M.~Restelli, A.~Pecorino, and D.~Calandriello.
\newblock Safe policy iteration.
\newblock In S.~Dasgupta and D.~McAllester, editors, \emph{Proceedings of the
  30th International Conference on Machine Learning}, volume~28 of
  \emph{Proceedings of Machine Learning Research}, pages 307--315, Atlanta,
  Georgia, USA, 17--19 Jun 2013. PMLR.
\newblock URL \url{https://proceedings.mlr.press/v28/pirotta13.html}.

\bibitem[Raffin et~al.(2021)Raffin, Hill, Gleave, Kanervisto, Ernestus, and
  Dormann]{stable-baselines3}
A.~Raffin, A.~Hill, A.~Gleave, A.~Kanervisto, M.~Ernestus, and N.~Dormann.
\newblock Stable-baselines3: Reliable reinforcement learning implementations.
\newblock \emph{Journal of Machine Learning Research}, 22\penalty0
  (268):\penalty0 1--8, 2021.
\newblock URL \url{http://jmlr.org/papers/v22/20-1364.html}.

\bibitem[Ross and Bagnell(2010)]{pmlr-v9-ross10a}
S.~Ross and D.~Bagnell.
\newblock Efficient reductions for imitation learning.
\newblock In Y.~W. Teh and M.~Titterington, editors, \emph{Proceedings of the
  Thirteenth International Conference on Artificial Intelligence and
  Statistics}, volume~9 of \emph{Proceedings of Machine Learning Research},
  pages 661--668, Chia Laguna Resort, Sardinia, Italy, 13--15 May 2010. PMLR.
\newblock URL \url{https://proceedings.mlr.press/v9/ross10a.html}.

\bibitem[Ross et~al.(2010)Ross, Gordon, and
  Bagnell]{DBLP:journals/corr/abs-1011-0686}
S.~Ross, G.~J. Gordon, and J.~A. Bagnell.
\newblock No-regret reductions for imitation learning and structured
  prediction.
\newblock \emph{CoRR}, abs/1011.0686, 2010.
\newblock URL \url{http://arxiv.org/abs/1011.0686}.

\bibitem[Rusu et~al.(2016)Rusu, Colmenarejo, G{\"{u}}l{\c{c}}ehre, Desjardins,
  Kirkpatrick, Pascanu, Mnih, Kavukcuoglu, and
  Hadsell]{DBLP:journals/corr/RusuCGDKPMKH15}
A.~A. Rusu, S.~G. Colmenarejo, {\c{C}}.~G{\"{u}}l{\c{c}}ehre, G.~Desjardins,
  J.~Kirkpatrick, R.~Pascanu, V.~Mnih, K.~Kavukcuoglu, and R.~Hadsell.
\newblock Policy distillation.
\newblock In Y.~Bengio and Y.~LeCun, editors, \emph{4th International
  Conference on Learning Representations, {ICLR} 2016, San Juan, Puerto Rico,
  May 2-4, 2016, Conference Track Proceedings}, 2016.
\newblock URL \url{http://arxiv.org/abs/1511.06295}.

\bibitem[Sallab et~al.(2017)Sallab, Abdou, Perot, and Yogamani]{Sallab_2017}
A.~E. Sallab, M.~Abdou, E.~Perot, and S.~Yogamani.
\newblock Deep reinforcement learning framework for autonomous driving.
\newblock \emph{Electronic Imaging}, 29\penalty0 (19):\penalty0 70–76, Jan.
  2017.
\newblock ISSN 2470-1173.
\newblock \doi{10.2352/issn.2470-1173.2017.19.avm-023}.
\newblock URL \url{http://dx.doi.org/10.2352/ISSN.2470-1173.2017.19.AVM-023}.

\bibitem[Silver et~al.(2016)Silver, Huang, Maddison, Guez, Sifre, van~den
  Driessche, Schrittwieser, Antonoglou, Panneershelvam, Lanctot, Dieleman,
  Grewe, Nham, Kalchbrenner, Sutskever, Lillicrap, Leach, Kavukcuoglu, Graepel,
  and Hassabis]{44806}
D.~Silver, A.~Huang, C.~J. Maddison, A.~Guez, L.~Sifre, G.~van~den Driessche,
  J.~Schrittwieser, I.~Antonoglou, V.~Panneershelvam, M.~Lanctot, S.~Dieleman,
  D.~Grewe, J.~Nham, N.~Kalchbrenner, I.~Sutskever, T.~Lillicrap, M.~Leach,
  K.~Kavukcuoglu, T.~Graepel, and D.~Hassabis.
\newblock Mastering the game of go with deep neural networks and tree search.
\newblock \emph{Nature}, 529:\penalty0 484--503, 2016.
\newblock URL
  \url{http://www.nature.com/nature/journal/v529/n7587/full/nature16961.html}.

\bibitem[Sutton and Barto(2018)]{sutton1998reinforcement}
R.~S. Sutton and A.~G. Barto.
\newblock \emph{Reinforcement learning: An introduction}.
\newblock MIT press, 2018.

\bibitem[Sutton et~al.(1999)Sutton, McAllester, Singh, and
  Mansour]{10.5555/3009657.3009806}
R.~S. Sutton, D.~McAllester, S.~Singh, and Y.~Mansour.
\newblock Policy gradient methods for reinforcement learning with function
  approximation.
\newblock In \emph{Proceedings of the 13th International Conference on Neural
  Information Processing Systems}, NIPS'99, page 1057–1063, Cambridge, MA,
  USA, 1999. MIT Press.

\bibitem[Towers et~al.(2024)Towers, Kwiatkowski, Terry, Balis, De~Cola, Deleu,
  Goul{\~a}o, Kallinteris, Krimmel, KG, et~al.]{towers2024gymnasium}
M.~Towers, A.~Kwiatkowski, J.~Terry, J.~U. Balis, G.~De~Cola, T.~Deleu,
  M.~Goul{\~a}o, A.~Kallinteris, M.~Krimmel, A.~KG, et~al.
\newblock Gymnasium: A standard interface for reinforcement learning
  environments.
\newblock \emph{arXiv preprint arXiv:2407.17032}, 2024.

\bibitem[Verma et~al.(2018)Verma, Murali, Singh, Kohli, and
  Chaudhuri]{DBLP:journals/corr/abs-1804-02477}
A.~Verma, V.~Murali, R.~Singh, P.~Kohli, and S.~Chaudhuri.
\newblock Programmatically interpretable reinforcement learning.
\newblock \emph{CoRR}, abs/1804.02477, 2018.
\newblock URL \url{http://arxiv.org/abs/1804.02477}.

\bibitem[Williams(1992)]{10.1007/BF00992696}
R.~J. Williams.
\newblock Simple statistical gradient-following algorithms for connectionist
  reinforcement learning.
\newblock \emph{Mach. Learn.}, 8\penalty0 (3–4):\penalty0 229–256, May
  1992.
\newblock ISSN 0885-6125.
\newblock \doi{10.1007/BF00992696}.
\newblock URL \url{https://doi.org/10.1007/BF00992696}.

\bibitem[Xing et~al.(2023)Xing, Nagata, Zou, Neftci, and Krichmar]{XING2023228}
J.~Xing, T.~Nagata, X.~Zou, E.~Neftci, and J.~L. Krichmar.
\newblock Achieving efficient interpretability of reinforcement learning via
  policy distillation and selective input gradient regularization.
\newblock \emph{Neural Networks}, 161:\penalty0 228--241, 2023.
\newblock ISSN 0893-6080.
\newblock \doi{https://doi.org/10.1016/j.neunet.2023.01.025}.
\newblock URL
  \url{https://www.sciencedirect.com/science/article/pii/S0893608023000254}.

\end{thebibliography}

\end{document}